\documentclass[12pt]{article}
\usepackage{fullpage}

\usepackage[utf8]{inputenc} 
\usepackage[T1]{fontenc}    
\usepackage{hyperref}       
\usepackage{url}            
\usepackage{booktabs}       
\usepackage{amsfonts}       
\usepackage{nicefrac}       
\usepackage{microtype}      

\usepackage{amsfonts,amsmath,amssymb,amsthm,graphicx}
\usepackage{changepage}
\usepackage{enumerate}
\usepackage[ruled,vlined]{algorithm2e}
\usepackage{bbm}
\usepackage{comment}

\usepackage{scalerel}
\usepackage{accents}

\usepackage{todonotes}

\usepackage{natbib}
\usepackage{csquotes}
\usepackage{comment}

\usepackage{cleveref}

\theoremstyle{plain}
\newtheorem{theorem}{Theorem}[section]
\newtheorem{lemma}[theorem]{Lemma}

\newtheorem{corollary}[theorem]{Corollary}

\theoremstyle{plain}
\newtheorem{definition}{Definition}[section] 

\newtheorem{assumption}[definition]{Assumption}

\usepackage{amssymb}
\usepackage{amsmath}
\usepackage{ifthen}
\usepackage{xfrac}
\usepackage{sidecap}
\usepackage{subfig}
\usepackage{caption}

\DeclareMathAlphabet{\mathpzc}{OT1}{pzc}{m}{it}

\newcommand{\newagentvar}[3][\noaccents]{%
\expandafter\newcommand\expandafter{\csname #2\endcsname}{#1{#3}}%
\expandafter\newcommand\expandafter{\csname #2s\endcsname}{#1{\boldsymbol{#3}}}%
\expandafter\newcommand\expandafter{\csname #2smi\endcsname}[1][i]{#1{\boldsymbol{#3}}_{-##1}}%
\expandafter\newcommand\expandafter{\csname #2i\endcsname}[1][i]{#1{#3}\agind[##1]}%
\expandafter\newcommand\expandafter{\csname #2ith\endcsname}[1][i]{#1{#3}_{(##1)}}%
}

\newcommand{\newitemvar}[3][\noaccents]{%
\expandafter\newcommand\expandafter{\csname #2\endcsname}{#1{#3}}%
\expandafter\newcommand\expandafter{\csname #2s\endcsname}{#1{\boldsymbol{#3}}}%
\expandafter\newcommand\expandafter{\csname #2smj\endcsname}[1][j]{#1{\boldsymbol{#3}}_{-##1}}%
\expandafter\newcommand\expandafter{\csname #2j\endcsname}[1][j]{#1{#3}_{##1}}%
\expandafter\newcommand\expandafter{\csname #2jth\endcsname}[1][j]{#1{#3}_{(##1)}}%
}

\newagentvar{alloc}{x}

\newagentvar{quant}{q}
\newagentvar[\constrained]{exquant}{\quant}
\newagentvar[\constrained]{critquant}{\quant}  
\newagentvar[\optconstrained]{monoq}{\quant}
\newagentvar{Val}{\nu}
\newagentvar{toquant}{Q}

\newcommand{\reward}{r}
\newcommand{\context}{x}
\newcommand{\hidden}{\theta}
\newcommand{\empirical}{\hat{\hidden}}

\newcommand{\bestcontext}{\context^*}
\newcommand{\chosencontext}{\context_t}

\newcommand{\dimension}{d}
\newcommand{\confidence}{\delta}
\newcommand{\xSet}{D}
\newcommand{\partition}{K}
\newcommand{\empiricaldiff}{\hat{\truediff}}
\newcommand{\truediff}{\Delta}

\newcommand{\parameter}{\zeta}
\newcommand{\epoch}{m}
\newcommand{\num}{n}
\newcommand{\totalnum}{N}
\newcommand{\horizon}{T}
\newcommand{\corruption}{C}
\newcommand{\corruptionf}{c}

\newcommand{\exploreSet}{S}
\newcommand{\explore}{s}

\newcommand{\exploreRatio}{\gamma}

\newcommand{\norm}[1]{\left\lVert#1\right\rVert}

\newcommand{\noise}{\eta}

\newcommand{\event}{\mathcal{E}}

\newcommand{\regret}{\mathcal{R}}
\newcommand{\extreme}{P}

\newcommand{\rv}{\rho}
\newcommand{\bestimpirical}{\context^{(\epoch)}_*}
\newcommand{\secondimpirical}{\context^{(\epoch)}_{(2)}}
\newcommand{\secondbest}{\context_{(2)}}

\newcommand{\estimationdiff}{\beta}
\newcommand{\sumerror}{\rho}

\newcommand{\ha}{\hat{a}}

\newcommand{\ellipsoid}{E}
\newcommand{\numExplore}{\num_e}

\DeclareMathOperator*{\argmax}{arg\,max}

\newcommand{\given}{\,\mid\,}

\newcommand{\prob}[2][]{\text{\bf Pr}\ifthenelse{\not\equal{}{#1}}{_{#1}}{}\!\left[{\def\givenn{\middle|}#2}\right]}
\newcommand{\expect}[2][]{\text{\bf E}\ifthenelse{\not\equal{}{#1}}{_{#1}}{}\!\left[{\def\givenn{\middle|}#2}\right]}
\newcommand{\variance}[2][]{\text{\bf Var}\ifthenelse{\not\equal{}{#1}}{_{#1}}{}\!\left[{\def\givenn{\middle|}#2}\right]}

\newcommand{\tparen}{\big}
\newcommand{\tprob}[2][]{\text{\bf Pr}\ifthenelse{\not\equal{}{#1}}{_{#1}}{}\tparen[{\def\given{\tparen|}#2}\tparen]}
\newcommand{\texpect}[2][]{\text{\bf E}\ifthenelse{\not\equal{}{#1}}{_{#1}}{}\tparen[{\def\given{\tparen|}#2}\tparen]}

\newcommand{\sprob}[2][]{\text{\bf Pr}\ifthenelse{\not\equal{}{#1}}{_{#1}}{}[#2]}
\newcommand{\sexpect}[2][]{\text{\bf E}\ifthenelse{\not\equal{}{#1}}{_{#1}}{}[#2]}

\newcommand{\indicator}[1]{\mathbbm{1}\left[ #1 \right]}

\newcommand{\abs}[1]{\left| #1 \right |}
\newcommand{\bigO}[1]{O \left( #1 \right )}

\newcommand{\slr}[1]{\left( #1 \right)}
\newcommand{\mlr}[1]{\left[ #1 \right]}
\newcommand{\blr}[1]{\left\{ #1 \right\}}
\newcommand{\inner}[2]{\langle #1,#2 \rangle}

\newcommand{\reals}{\mathbb{R}}

\let\oldparagraph\paragraph
\renewcommand{\paragraph}[1]{\oldparagraph{#1.}}

\begin{document}


\title{Stochastic Linear Optimization with Adversarial Corruption}
\author{Yingkai Li\thanks{Department of Computer Science, Northwestern University.
Email: \texttt{yingkai.li@u.northwestern.edu}.}
\and Edmund Y. Lou\thanks{Department of Economics, Northwestern University.
Email: \texttt{edmund.lou@u.northwestern.edu}.}
\and Liren Shan\thanks{Department of Computer Science, Northwestern University.
Email: \texttt{lirenshan2023@u.northwestern.edu}.}}
\date{\today}

\maketitle

\begin{abstract}
We extend the model of stochastic bandits with adversarial corruption \citep{lykouris2018stochastic} to the stochastic linear optimization problem \citep{dani2008stochastic}. 
Our algorithm is agnostic to the amount of corruption chosen by the adaptive adversary. 
The regret of the algorithm only increases linearly in the amount of corruption.
Our algorithm involves using L\"{o}wner-John's ellipsoid 
for exploration 
and dividing time horizon into epochs 
with exponentially increasing size 
to limit the influence of corruption.

\end{abstract}
\section{Introduction}
\label{sec:intro}

The multi-armed bandit problem has been extensively studied in computer science, operations research and economics
since the seminal work of \cite{robbins1952some}. It is a model designed for sequential decision-making in which  
a player chooses at each time step amongst a finite set of available arms and receives a reward for the chosen decision. 
The player's objective is to minimize the difference, called regret, between the rewards she receives and the rewards accumulated
by the best arm. The rewards of each arm is drawn from a probability distribution in the stochastic multi-armed bandit problem; but 
in adversarial multi-armed bandit models, there is typically no assumption imposed on the sequence of rewards received by the player.

In recent work, \cite{lykouris2018stochastic} introduce a model in which 
an adversary could corrupt the stochastic reward generated by an arm pull. They provide an algorithm and show that the regret of 
this ``middle ground'' scenario degrades smoothly with the amount of corruption injected by the adversary.
\cite{gupta2019better} present an alternative algorithm which gives a significant improvement. 

With real-world applications such as fake reviews and effects of employing celebrity brand 
ambassadors in mind \citep{kapoor2019corruption},
we complement the literature by incorporating the notion of corruption into the stochastic linear optimization problem, and hence answering an open question suggested in \cite{gupta2019better}, in the framework of \cite{dani2008stochastic}.
In our finite-horizon model, the player chooses at each time step $t \leq \horizon$ a vector (i.e., an arm) in a fixed decision set $\xSet \subseteq \reals^\dimension$. 
To consider the problem dependent bound, we assume that $\xSet$ is a $\dimension$-dimensional polytope as in \cite{abbasi2011improved}. 
The regret of our algorithm is 
$\bigO{\frac{\dimension^{5/2} \corruption \log \horizon}{\truediff}
+ \frac{\dimension^6 
\log (\sfrac{\dimension \log \horizon}{\delta})
\log \horizon}{\truediff^2}}$, 
where $\truediff$ corresponds 
to the distance between the highest and lowest expected rewards, $C$ the amount of corruption and $\delta$ the level of confidence. In contrast to the stochastic model with corruption, 
our regret suffers an extra multiplicative loss of $\sfrac{1}{\truediff}$, which is caused by the separation of exploration and exploitation.

\subsection{Related works}
The finite-arm version of the stochastic linear optimization problem is introduced in \citet{auer2002using}. 
When the number of arms becomes infinity, 
the CONFIDENCEBALL algorithm \citep{dani2008stochastic}
obtains the worst case regret bound of $\bigO{d\sqrt{\horizon \log^3 \horizon}}$. \citet{li2019tight} improve this result by replacing $\log^2 \horizon$ by a $\log \log \horizon$ dependence.
For the problem dependent bound,  \citet{abbasi2011improved} 
show that the regret of their OFUL algorithm is 
$\bigO{\frac{\log (\sfrac{1}{\delta})}{\truediff} 
(\log \horizon + \dimension \log \log \horizon)^2}$, 
and
our algorithm achieves at least the same asymptotic performance when there exists an $\bigO{\log \horizon}$ amount of corruption. 
Similar to the result of \cite{lykouris2018stochastic}, 
both the CONFIDENCEBALL algorithm and the OFUL algorithm 
suffer linear $\Omega(T)$ regret even when the amount of corruption appears to be small.


There also have been works that strive to achieve good regret guarantees in both stochastic multi-armed bandit models and their adversarial counterparts, commonly known as ``the best of both worlds'' (e.g., \cite{bubeck2012best} and \cite{zimmert2018optimal}). In those algorithms the regret does not degrade smoothly as the amount of adversarial corruption increases. \cite{kapoor2019corruption} consider the corruption setting in the linear contextual bandit problem under a strong assumption that at each time step the adversary corrupts the data with a constant probability.


Our algorithm builds on \citet{gupta2019better}. 
To eliminate the effect from corruption, 
we borrow the idea of dividing the time horizon into epochs which increase exponentially in length 
and use only the estimation from the previous epoch to conduct exploitation in the current round. 
This approach weakens the dependence of current estimate on the levels of earlier corruption, 
so the negative impact from the adversary fades away over time. The main challenge of our paper is that 
we cannot simply adopt the widely used ordinary least square estimator since the correlation between different time steps of estimation impedes the application of concentration inequalities. We thus conduct exploration on each coordinate independently.


\section{Preliminaries}
\label{sec:prelim}

Let $\xSet \subseteq \reals^\dimension$ be a $\dimension$-polytope.
At each time step $t \in [\horizon] := \{1, 2, \dots, \horizon\}$,
the algorithm chooses an action $\chosencontext \in \xSet$.
Let $\hidden \in \reals^\dimension$ be an unknown hidden vector and $\{\noise_t\}$ a sequence of sub-Gaussian random noise with mean 0 and variance proxy 1. For a given time step, $t$, and a chosen action, $\context_t$, we define the reward as $\reward_t(\context_t) = \inner{\context_t}{\hidden} + \noise_t$, where the first term is the inner product of $\chosencontext$ and $\hidden$.
We assume without loss of generality that 
$\norm{\hidden}_2 \leq 1$ and 
$\norm{\context}_2 \leq 1$ for all $\context \in \xSet$. 

At each time step $t \leq \horizon$, 
there is an adaptive adversary who may corrupt 
the observed reward 
by choosing 
a corruption function 
$\corruptionf_t:\xSet \to [-1, 1]$. 
The algorithm chooses first $\chosencontext$, 
then observes the corrupted reward 
$\reward_t(\chosencontext) + \corruptionf_t(\chosencontext)$, 
and finally receives the actual reward 
$\reward_t(\chosencontext)$. 
We denote by $\corruption = \sum_{t = 1}^T 
\max_{\context \in \xSet} 
\abs{\corruptionf_t(\context)}$ the total corruption generated by the adversary. The value of $\corruption$ is unknown to the algorithm, which is, in turn, evaluated by \textit{pseudo-regret}:
\begin{align*}
\regret(\horizon) = 
\sum_{t=1}^{\horizon} 
\inner{\bestcontext - \chosencontext}{\hidden},
\end{align*}
where $\bestcontext$ is an action that  
maximizes the expected reward. 
In this paper, we assume that $\bestcontext$ 
is unique.\footnote{This assumption is without loss of generality because it is of probability 1 that the best action is unique when the action set is perturbed with a random noise.}
Let $\extreme$ be the set of extreme points of $\xSet$ and $\extreme^- = \extreme \backslash \{\bestcontext\}$. 
The extreme point that generates the second highest reward is denoted $\secondbest$; i.e., 
$\secondbest = \argmax_{\context \in \extreme^-}
{\inner{\bestcontext - \context}{\hidden}}$. 
Thus the corresponding expected reward gap between $\bestcontext$ and $\secondbest$ is given by 
\begin{align*}
\truediff = 
\inner{\bestcontext - \secondbest}{\hidden}. 
\end{align*}

We now introduce the so-called L{\"o}wner-John ellipsoid (see \citet{grtschel1988geometric} for a detailed discussion), 
which plays a key role in the construction of our algorithm. 
\begin{theorem}[L{\"o}wner-John's Ellipsoid Theorem]
\label{thm:john theorem origin}
For any bounded convex body $K \subseteq \reals^\dimension$,
there exists an ellipsoid $E$ satisfying
\begin{align*}
    \ellipsoid \subseteq K \subseteq \dimension \ellipsoid.
\end{align*}
\end{theorem}

A discussion of finding efficiently the L{\"o}wner-John ellipsoid is deferred in \Cref{sec:efficiency}. Let $E \subseteq D$ be a L{\"o}wner-John ellipsoid guaranteed by \Cref{thm:john theorem origin}. 
Let $\explore_0$ be the center and $\explore_j$ the $j$-th principal axis, $j \in [d]$, of $E$. 
Without loss of generality, we assume that $\explore_0$ is the origin; otherwise we could shift the origin toward $\explore_0$ such that the new decision set $\xSet^\prime = \xSet - s_0$. 
Then the reward for each action is shifted by the same constant, 
and therefore the problem remains unchanged. In what follows, we dub $\exploreSet = \blr{\explore_1,\cdots,\explore_\dimension}$ the \textit{exploration set}. 
It is worth noting that $\exploreSet$ corresponds to an orthogonal basis for $\xSet$.
From \Cref{thm:john theorem origin}, we obtain the following result.
\begin{corollary}
\label{cor:johnellipse}
For each $\context \in \xSet$, 
we have $\context = \sum_{j=1}^{\dimension} \ha_j \explore_j$, 
where $\abs{\ha_j} \leq 2\dimension$. 
\end{corollary}

\begin{algorithm}[t]
\SetAlgoLined
\caption{SBE: Support Basis Exploration Algorithm}
\label{alg:explobycoor}
 Parameters: 
 Confidence $\confidence \in (0,1)$, 
 time horizon $\horizon$, 
 decision set $\xSet$. 
 
 Initialization: 
 Exploration set $\exploreSet = \{\explore_j\}_{j \in [\dimension]}$.
 
 
 Set $\parameter = 2^{14} \dimension^6 
\log (\sfrac{4\dimension \log \horizon}{\delta})$. 
 Set estimated gap 
 $\empiricaldiff^{(0)} = 1$ 
 and exploration ratio $\exploreRatio_{0} = \sfrac{1}{5}$. 
 
\For{epoch $\epoch=1,2,\cdots, M$}{ 

    Set $\num_\epoch = \parameter \cdot 4^\epoch$.
    Let $\totalnum_\epoch = \num_\epoch + \parameter (\empiricaldiff^{(\epoch-1)})^{-2}$, 
    and $\horizon_\epoch 
    = \horizon_{\epoch-1} 
    + \totalnum_\epoch$. 
    
\For{$t$ from $\horizon_{\epoch-1} + 1$ to $\horizon_\epoch$}{

    \eIf{$Z = 1$ \text{for Bernoulli random variable} $Z \sim \mathrm{Bernoulli}(\gamma_{\epoch-1})$}{Sample uniformly an action from the exploration set $\exploreSet$.}
    {Choose the best action $\context_*^{(\epoch-1)}$
    according to the estimate $\empirical^{(\epoch-1)}$.}
    
}
Let $\empirical^{(\epoch)}$ be the estimate
of $\theta$ in this epoch, 
defined later in \Cref{sec:parameter}. 

Set $\empiricaldiff^{(\epoch)}$ as the 
maximum of $2^{-\epoch}$ 
and the difference between the expected reward for the best and second best actions given $\empirical^{(\epoch)}$.

Set $\exploreRatio_\epoch = \sfrac{(\empiricaldiff^{(\epoch)})^{-2}}{\slr{(\empiricaldiff^{(\epoch)})^{-2} + 2^{2(\epoch+1)}}}$.
}
\end{algorithm}
\section{The SBE algorithm}
\label{sec:general}

In this section, we introduce our Support Basis Exploration (SBE) algorithm
for the stochastic linear optimization problem with adversarial corruption (see \Cref{alg:explobycoor}).

The algorithm runs in epochs which increase exponentially in length. 
Each epoch $\epoch$ has a length greater than $4^\epoch$, and therefore the total number of epochs $M$ is bounded above by $\log \horizon$. 
The choice of current action depends only on information received from the last epoch, 
so the level of earlier corruption will have a decreasing effect on later epochs.
Different from other algorithms for stochastic linear optimization models, we separate exploration and exploitation so that we can decrease the correlation between vector pulls in each epoch and thus minimize the influence of adversarial corruption on the estimate. This approach will inevitably increase the regret
by a multiplicative $\sfrac{1}{\truediff}$ factor.


Given the exploration set $\exploreSet$ defined in Section \ref{sec:prelim}, 
we can represent each vector in the decision set~$\xSet$ according to the elements of $\exploreSet$. By \Cref{cor:johnellipse}, the coefficient on each coordinate, in this new representation, is bounded by $2\dimension$. 
It follows that the maximal projection on the basis vector $\explore_j$ is simply $2d\cdot\explore_j$. 
In other words, $\explore_j$ contains the maximum information up to a constant $2\dimension$ in its own direction. 
Since basis vectors $\explore_j$ and $\explore_k$ are orthogonal to each other, 
there is no information loss using the exploration set $\exploreSet$ in the algorithm. 
Thus, we obtain a better concentration in each round of estimation.
Note that our algorithm can take any basis $S$ as input that has similar performance as in \Cref{cor:johnellipse}, 
and in \Cref{sec:efficiency}, 
we provide an efficient algorithm that finds such a set with a multiplicative loss $\dimension$ in regret. 
The construction of other parameters in the algorithm
is explained in the next section.

\section{Parameter estimation}
\label{sec:parameter}

We now know that the hidden vector, $\hidden$, can be represented according to the exploration set $\exploreSet$; 
that is, $\hidden = \sum_{j=1}^\dimension b_j \explore_j$. 
For any $j \in [\dimension]$, let $\xi_{j}^t$ be an indicator defined on the event if the basis vector $\explore_j$ is chosen in time step $t$.
Let $\numExplore^{(\epoch)} = \expect{\sum_{t = T_{\epoch-1}+1}^{T_\epoch} \xi_{j}^t}$ 
be the expected number of time steps used to explore each basis vector $\explore_j$.
But since $\explore_j$ is sampled uniformly, it follows that $\numExplore^{(\epoch)}$ is independent of $j$. 
Then, the ``average reward"
for exploring $\explore_j$ in epoch $\epoch$ is\footnote{
This is not the actual average reward
as $\numExplore^{(\epoch)}$ is not the realized number of time steps used to explore $\explore_j$.} 

\begin{align*}
    \reward^{(\epoch)}_{j} 
    = \frac{1}{\numExplore^{(\epoch)}} 
    \sum_{t= \horizon_{m-1} + 1}^{\horizon_m} 
    \xi_{j}^t \cdot \slr{\inner{\explore_j }{\hidden} 
    + \noise_t + \corruptionf_t(\explore_j)}.
\end{align*}

Note that $\xi_j^t$ 
is independent of the noise, $\noise_t$, as well as the amount of corruption, $\corruptionf_t(\explore_j)$,
taking expectation over the randomness 
of independent variables $\xi_j^t$ and $\noise_t$
on both sides yields
\begin{align*}
    \expect{\reward^{(\epoch)}_{j}} = \inner{\explore_j}{\hidden} 
    + \frac{1}{\totalnum_\epoch} \sum_{t= \horizon_{m-1} + 1}^{\horizon_m} \expect{\corruptionf_t(\explore_j)}
    \leq b_j \norm{\explore_j}^2_2 
    + \frac{\corruption_\epoch}{\totalnum_\epoch},
\end{align*}
where $\corruption_\epoch = \sum_{t= \horizon_{m-1} + 1}^{\horizon_m} 
\max_{\context \in \xSet} 
\abs{\corruptionf_t(\context)}$. 
At the end of each epoch $\epoch$, we have 
$\hat{b}^{(\epoch)}_j = 
\frac{\reward^{(\epoch)}_{j}}{\norm{\explore_j}^2_2}$ 
as the estimate of $b_j$ and
$\empirical^{(\epoch)} = 
\sum_{j=1}^\dimension \hat{b}^{(\epoch)}_j \explore_j$
as the estimate of~$\hidden$. 
Before giving an uniform bound for the error in expected reward $\inner{x}{\empirical^{(\epoch)}-\hidden}$, we provide first an upper bound for the error 
of $\empirical^{(\epoch)}$ in each dimension $j$.


\subsection{Error of estimated reward}
\begin{lemma}
\label{lem:estimate theta}
With probability at least $1-\confidence$, 
the estimate $\hat{b}_j^{(\epoch)}$ is such that
\begin{align*}
    \abs{\hat{b}_j^{(\epoch)} - b_j}\norm{\explore_j}^2_2 
    \leq \frac{2\corruption_\epoch}{\totalnum_\epoch} 
    + \frac{\empiricaldiff^{(\epoch-1)}}{32\dimension^2}
\end{align*}
for all $j \in [\dimension]$ and for all epoch $\epoch \in [M]$.
\end{lemma}

\begin{proof}
Since the indicator $\xi_j^t$ 
and the noise $\noise_t$ are independent random variables, by a form of the Chernoff-Hoeffding bound in \citet{hoeffding1963probability}, we have for any deviation $\kappa$ and any $j \in [\dimension]$
\begin{align}\label{eq:concentration1}
\prob{\abs{\frac{1}{\numExplore^{(\epoch)}}
\sum_{t= \horizon_{m-1} + 1}^{\horizon_m} \xi_{j}^t \cdot \slr{\inner{\explore_j}{\hidden}+\noise_t} - \inner{\explore_j}{\hidden}
} 
\geq \frac{\kappa}{2}} 
\leq 2\exp\blr{-\frac{\kappa^2 \num_e^{(\epoch)}}{16}}.
\end{align}

For any $j \in [\dimension]$, let 
$X_t = \slr{\xi_j^t - \sfrac{\numExplore^{(\epoch)}}{\totalnum_\epoch}}
\corruptionf_t(\explore_j)$ 
for all $t$. 
Denote by $\blr{\mathcal{F}_t}_{t=1}^T$ 
the filtration generated by random variables 
$\blr{\xi_j^s}_{j \in [\dimension], s\leq t}$ 
and $\blr{\eta_s}_{s \leq t+1}$, and 
define $Y_t = \sum_{s=1}^t X_s$. 
Since $\xi_j^t$ is independent of the corruption level $\corruptionf_j^t$ 
conditional on $\mathcal{F}_{t-1}$, 
$\blr{Y_t}_{t=1}^T$ yields a martingale with respect to the filtration $\blr{\mathcal{F}_t}$. 
The variance of $X_t$ conditional on $\mathcal{F}_{t-1}$ can be bounded as
\begin{align}
\label{eq:variance bound}
V = \expect{X_t^2|\mathcal{F}_{t-1}} \leq \sum_{t = \horizon_{\epoch-1}+1}^{\horizon_\epoch} \abs{\corruptionf_t(\explore_j)}\variance{\xi^t_j} 
\leq \frac{\numExplore^{(\epoch)}}{\totalnum_\epoch} \sum_{t=\horizon_{\epoch-1}+1}^{\horizon_\epoch} \abs{\corruptionf_t(\explore_j)}.
\end{align}
The first inequality holds because 
$\abs{\corruptionf_t(\explore_j)} \leq 1$,
and the second inequality holds because 
$\variance{\xi^t_j} \leq
\frac{\numExplore^{(\epoch)}}{\totalnum_\epoch}$. 
Using a Freedman-type concentration inequality for martingales \citep{beygelzimer2011contextual}, 
we have for any $\nu > 0$, 
\begin{align*}
\prob{
\frac{1}{\numExplore^{\epoch}}\sum_{t=\horizon_{\epoch-1}+1}^{\horizon_\epoch} X_t 
\geq 
\frac{V+ \ln\sfrac{4}{\nu}}{\numExplore^{(\epoch)}}} 
\leq \frac{\nu}{4}.
\end{align*}
Note that 
$\frac{1}{\numExplore^{\epoch}}\sum_{t=\horizon_{\epoch-1}+1}^{\horizon_\epoch} X_t
= \frac{1}{\numExplore^{(\epoch)}}\sum_{t=\horizon_{\epoch-1}+1}^{\horizon_m} \xi_j^t c_t(s_j)
- \frac{1}{\totalnum_\epoch}\sum_{t=\horizon_{\epoch-1}+1}^{\horizon_m}  c_t(s_j)$. 
Combining it with Equation \eqref{eq:variance bound}, 
for any $\nu > 0$, 
we have
\begin{align*}
\prob{\frac{\sum_{t=\horizon_{\epoch-1}+1}^{\horizon_m} \xi_j^t c_t(s_j)}{\numExplore^{(\epoch)}} \geq 
\frac{2\corruption_\epoch}{\totalnum_\epoch}
+ \frac{\ln \sfrac{4}{\nu}}{\numExplore^{(\epoch)}}}
\leq 
\prob{
\frac{1}{\numExplore^{\epoch}}\sum_{t=\horizon_{\epoch-1}+1}^{\horizon_\epoch} X_t 
\geq 
\frac{V+ \ln\sfrac{4}{\nu}}{\numExplore^{(\epoch)}}} 
\leq \frac{\nu}{4}.
\end{align*}
For any $0 < \kappa < 1$, 
substituting $\nu = 4\exp\blr{-\frac{\kappa\numExplore^{(\epoch)}}{2}}$, we can get
\begin{align*}
    \prob{\frac{\sum_{t=\horizon_{\epoch-1}+1}^{\horizon_m} \xi_j^t c_t(s_j)}{\numExplore^{(\epoch)}} \geq 
    \frac{\kappa}{2} + \frac{2\corruption_\epoch}{\totalnum_\epoch}} \leq \exp\blr{-\frac{\kappa\numExplore^{(\epoch)}}{2}}.
\end{align*}
Similarly, consider the sequence $\blr{-X_t}$. Then, for any $0 < \kappa < 1$, we have 
\begin{align}\label{eq:concentration2}
\prob{\abs{\frac{\sum_{t=\horizon_{\epoch-1}+1}^{\horizon_m} \xi_j^t c_t(s_j)}{\numExplore^{(\epoch)}}}\geq \frac{\kappa}{2}+\frac{2\corruption_\epoch}{\totalnum_\epoch}}
\leq 2\exp\blr{-\frac{\kappa\numExplore^{(\epoch)}}{2}} 
\leq 2\exp\blr{-\frac{\kappa^2\numExplore^{(\epoch)}}{16}}.
\end{align}
Combining Inequalities \eqref{eq:concentration1} and \eqref{eq:concentration2} yields
\begin{align*}
    \prob{\abs{\reward_j^{(\epoch)}-\inner{\explore_j}{\hidden}}\geq \kappa+\frac{2\corruption_\epoch}{\totalnum_\epoch}} \leq 4\exp\blr{-\frac{\kappa^2\numExplore^{(\epoch)}}{16}}.
\end{align*}
Let $\kappa = \frac{\empiricaldiff^{(\epoch-1)}}{32\dimension^2} < 1$ and
$\parameter = 2^{14} \dimension^5 
\log (\sfrac{4\dimension \log \horizon}{\delta})$. 
Then 
$$\numExplore^{(\epoch)} 
= \frac{\parameter}{\dimension} (\empiricaldiff^{(\epoch-1)})^{-2}
= 2^{14} \dimension^4 (\empiricaldiff^{(\epoch-1)})^{-2}
\log (\sfrac{4\dimension \log \horizon}{\delta}),$$ 
and $\frac{\kappa^2\numExplore^{(\epoch)}}{16}
= \log\slr{\sfrac{4 \dimension \log \horizon}{\delta}}$.
It follows that
\begin{align*}
    &\prob{\abs{b^{(\epoch)}_j - b_j}\norm{\explore_j}^2_2 
    \geq \frac{2\corruption_\epoch}{\totalnum_\epoch} 
    + \frac{\empiricaldiff^{(\epoch-1)}}{32\dimension^2}}\\
    =\, & \prob{\abs{\reward^{(\epoch)}_{j} 
    - \inner{\explore_j}{\hidden} 
    }
    \geq \frac{2\corruption_\epoch}{\totalnum_\epoch}+
    \frac{\empiricaldiff^{(\epoch-1)}}{32\dimension^2}
    }
    \leq \frac{\confidence}{\dimension \log{\horizon}},
\end{align*}
where the first equality holds because
by the definition of $b^{(\epoch)}_j$
and $b_j$, 
$b^{(\epoch)}_j \norm{\explore_j}^2_2 
= \reward^{(\epoch)}_{j}$, 
and $b_j \norm{\explore_j}^2_2 
= \inner{\explore_j}{\hidden}$. 
By applying the union bound for all $j\in [\dimension]$ and epoch $\epoch \in [M]$, 
we obtain the desired result.
\end{proof}

\begin{lemma}
\label{lem:confidence interval}
With probability at least $1-\confidence$, we have
\begin{align*}
    \abs{\inner{x}{\empirical^{(\epoch)}-\hidden}} \leq \frac{4\dimension^{2}\corruption_\epoch}{\totalnum_\epoch} + \frac{\empiricaldiff^{(\epoch-1)}}{16}.
\end{align*}
for all epochs $\epoch \in [M]$ and all $x \in \xSet$.
\end{lemma}
\begin{proof}
Since the exploration set $\exploreSet$ is an orthogonal set, for any context $x$, 
there exists multipliers $\{\ha_j\}_{j\in [d]}$ such that 
\begin{align*}
    \abs{\inner{x}{\empirical^{(\epoch)}-\hidden}} 
    = \sum_{j=1}^\dimension 
    \abs{\ha_j \inner{\explore_j}{\empirical^{(\epoch)}-\hidden}} 
    \leq \sum_{j=1}^\dimension 
    \abs{\ha_j} \abs{b^{(\epoch)}_j-b_j}\norm{\explore_j}^2_2.
\end{align*}
Then~\Cref{cor:johnellipse} and~\Cref{lem:estimate theta} together imply, with probability $1-\delta$, that
\begin{equation*}
    \abs{\inner{x}{\empirical^{(\epoch)}-\hidden}} 
    \leq \frac{4\dimension^{2}\corruption_\epoch}{\totalnum_\epoch} 
    + \frac{ \empiricaldiff^{(\epoch-1)}}{16}. \qedhere
\end{equation*}
\end{proof}
For simplicity, we denote
\begin{equation}\label{eq:beta}
\beta_\epoch = 
\frac{4\dimension^2 \corruption_\epoch}{\totalnum_\epoch}
+ \frac{\empiricaldiff^{(\epoch-1)}}{16}
\end{equation} 
and let $\event$ be the event that $\abs{\inner{x}{\empirical^{(\epoch)}-\hidden}} 
\leq \beta_\epoch ~\text{for all $\epoch$ and for all $\context$}$.
Note that event $\event$ happens with probability at least $1-\delta$.


\subsection{Bound analysis for estimated gap}
Let us now turn to provide the upper and lower bounds 
for the estimated gap $\empiricaldiff^{(\epoch)}$. 
Let $\bestimpirical \in 
\argmax_{\context \in \extreme}
\inner{\context}{\empirical^{(\epoch)}}$
be one of the actions 
that maximizes the expected reward given the estimate~$\empirical^{(\epoch)}$. 
We also define 
$\extreme^{\epoch-} = \extreme \backslash \{\bestimpirical\}$, 
and let the second best action 
given $\empirical^{(\epoch)}$ be 
$\secondimpirical = 
\argmax_{\context \in \extreme^{\epoch-}}
\inner{\context}{\empirical^{(\epoch)}}$. 
Since $\bestimpirical$ may not be unique, 
the expected reward for $\bestimpirical$ 
and $\secondimpirical$ may coincide.
Then the estimated gap in epoch $\epoch$ corresponds to
$\empiricaldiff^{(\epoch)} = 
\max\{ 2^{-\epoch}, 
\inner{\bestimpirical - \secondimpirical}{\empirical^{(\epoch)}}
\}$.

\begin{lemma}[Upper Bound for $\empiricaldiff^{(\epoch)}$]
\label{lem:upper bound}
Suppose that event $\event$ happens, 
then for all epochs $\epoch\geq 1$ 
\begin{align*}
    \empiricaldiff^{(\epoch)} \leq 2\mlr{\truediff + 2^{-\epoch} + 4\dimension^{2}\sum_{s=1}^{\epoch} \slr{\frac{1}{8}}^{m-s}\frac{\corruption_s}{\totalnum_s}}.
\end{align*}
\end{lemma}

\begin{proof}
First note that $\empiricaldiff^{(\epoch)} = 2^{-\epoch}$ whenever 
$\inner{\bestimpirical - \secondimpirical}{\empirical^{(\epoch)}} \leq 2^{-\epoch}$; 
otherwise we have a unique expected reward-maximizing action for estimate $\empirical^{(\epoch)}$.
By the uniqueness of $\bestcontext$, we have
$\inner{\bestimpirical-\bestcontext}{\hidden} \leq 0$, which implies that
\begin{align*}
\inner{\bestimpirical - \secondbest}{\hidden} 
= \inner{\bestimpirical - \bestcontext}{\hidden} 
+ \inner{\bestcontext - \secondbest}{\hidden} 
\leq \truediff,
\end{align*}
because $\inner{\bestcontext - \secondbest}{\hidden} 
\leq \truediff$. 
For the case $\secondbest \not= \bestimpirical$, 
we have $\secondbest \in \extreme^{\epoch-}$. It follows that $\inner{\secondbest-\secondimpirical}{\empirical^{(\epoch)}} \leq 0$, and therefore
\begin{align*}
    \empiricaldiff^{(\epoch)} = 
    \inner{\bestimpirical}{\empirical^{(\epoch)}-\hidden} 
    & + \inner{\bestimpirical - \secondbest}{\hidden}\\ 
    & + \inner{\secondbest}{\hidden-\empirical^{(\epoch)}} 
    + \inner{\secondbest-\secondimpirical}{\empirical^m} \leq \truediff + 2 \estimationdiff_\epoch.
\end{align*}
The last inequality follows from Lemma \ref{lem:confidence interval} because when the event $\event$ occurs, 
both inequalities 
$\abs{\inner{\bestimpirical}{\empirical^{(\epoch)}-\hidden}} \leq \beta_\epoch$ and 
$\abs{\inner{\secondbest}{\hidden-\empirical^{(\epoch)}}} \leq \beta_\epoch$ are satisfied.

Now for the case $\secondbest = \bestimpirical$, it is straightforward to see $\bestcontext \not = \bestimpirical$ from the fact $\bestcontext \neq \secondbest$. 
This implies that the expected reward of $\secondimpirical$ 
given the estimate $\empirical^{(\epoch)}$ is at least as large as that of $\bestcontext$;
i.e., 
$\inner{\bestcontext - \secondimpirical}{\empirical^{(\epoch)}} \leq 0$. 
Therefore
\begin{align*}
    \empiricaldiff^{(\epoch)} &\leq 
    \inner{\bestimpirical-\bestcontext}{\empirical^{(\epoch)}} \\
    &= \inner{\bestimpirical}{\empirical^{(\epoch)}-\hidden} 
    + \inner{\bestimpirical-\bestcontext}{\hidden} 
    + \inner{\bestcontext}{\hidden-\empirical^{(\epoch)}}
    \leq -\truediff + 2 \estimationdiff_\epoch.
\end{align*}
Combining all cases yields 
\begin{align}\label{eq:recurive}
\empiricaldiff^{(\epoch)} \leq \truediff + 2\estimationdiff_\epoch + 2^{-\epoch}.
\end{align}
Given the initial assignment $\empiricaldiff^{(0)} = 1$, 
we know that 
$\empiricaldiff^{(1)} \leq \truediff + \frac{8\dimension^2 \corruption_1}{\totalnum_1}
+ \frac{1}{8} + \frac{1}{2}$, 
satisfying \Cref{lem:upper bound}. 
Applying inequality \eqref{eq:recurive} recursively, 
we thus obtain 
\begin{align*}
    \empiricaldiff^{(\epoch)} 
    &\leq 
    \truediff 
    + \frac{8\dimension^{2} \corruption_\epoch}{\totalnum_\epoch} 
    + 2^{-\epoch}
    + \frac{1}{8} \empiricaldiff^{(\epoch-1)}
    \\
    &\leq \truediff 
    + \frac{8\dimension^{2} \corruption_\epoch}{\totalnum_\epoch} 
    + 2^{-\epoch}
    + \frac{1}{8}\mlr{\truediff + 2^{-(\epoch-1)}
    + 4\dimension^{2}\sum_{s=1}^{\epoch-1} 
    \slr{\frac{1}{8}}^{m-1-s}\frac{\corruption_s}{\totalnum_s}} \\
    &\leq 2\mlr{\truediff + 2^{-\epoch} 
    + 4\dimension^{2}\sum_{s=1}^{\epoch} 
    \slr{\frac{1}{8}}^{m-s}\frac{\corruption_s}{\totalnum_s}}. \qedhere
\end{align*} 
\end{proof}

\begin{lemma}[Lower Bound for $\empiricaldiff^{(\epoch)}$]
\label{lem:lower bound}
Suppose that event $\event$ happens, 
then for all epochs $\epoch$
\begin{align*}
    \empiricaldiff^{(\epoch)} 
    \geq \frac{\truediff}{2} 
    - 2^{-\epoch-1} 
    - 8\dimension^{2}\sum_{s=1}^{\epoch} 
    \slr{\frac{1}{8}}^{m-s}
    \frac{\corruption_s}{\totalnum_s}.
\end{align*}
\end{lemma}

\begin{proof}
We consider first the case that the best action for $\empirical^{(\epoch)}$ is unique. 

If $\bestimpirical \neq \bestcontext$, 
we know that $\inner{\bestimpirical - \secondbest}{\hidden} \leq 0$, 
and thus 
\begin{align*}
\truediff 
&= \inner{\bestcontext - \secondbest}{\hidden}
\leq \inner{\bestcontext - \bestimpirical}{\hidden} \\
&= \inner{\bestcontext - \bestimpirical}{\empirical^{(\epoch)}} 
+ \inner{\bestcontext - \bestimpirical}{\hidden - \empirical^{(\epoch)}} 
\leq 2\estimationdiff_\epoch 
\end{align*}
as the term $\inner{\bestcontext - \bestimpirical}{\empirical^{(\epoch)}}$
is always negative. 

If $\bestimpirical = \bestcontext$, 
then $\secondimpirical \neq \bestcontext$. It follows that 
\begin{align*}
\truediff 
&= \inner{\bestimpirical - \secondbest}{\hidden}
\leq \inner{\bestimpirical - \secondimpirical}{\hidden} \\
&= \inner{\bestimpirical - \secondimpirical}{\empirical^{(\epoch)}} 
+ \inner{\bestimpirical - \secondimpirical}{\hidden - \empirical^{(\epoch)}} \\
& \leq \empiricaldiff^{(\epoch)}
+ 2\estimationdiff_\epoch,  
\end{align*}
where the last inequality holds
because $\inner{\bestimpirical - \secondimpirical}{\empirical^{(\epoch)}} 
\leq \empiricaldiff^{(\epoch)}$. When the best action given $\empirical^{(\epoch)}$ is unique, we have
\begin{align*}
\empiricaldiff^{(\epoch)} \geq \truediff 
- 2\estimationdiff_\epoch.  
\end{align*}

For the case that the best action is not unique, 
let $\bestimpirical \neq \bestcontext$ be the best action given $\empirical^{(\epoch)}$. 
Then $\inner{\bestcontext-\bestimpirical}{\empirical^{(\epoch)}} \leq 0$,
giving that
\begin{align*}
    \truediff &\leq \inner{\bestcontext - \bestimpirical}{\hidden} 
    = \inner{\bestcontext}{\hidden - \empirical^{(\epoch)}} 
    + \inner{\bestcontext-\bestimpirical}{\empirical^{(\epoch)}} 
    + \inner{\bestimpirical}{\empirical^{(\epoch)}-\hidden}
    \leq 2\estimationdiff_\epoch.
\end{align*}
Now
\begin{align*}
    \empiricaldiff^{(\epoch)} \geq 0 \geq \truediff - 2\estimationdiff_\epoch.
\end{align*}
By applying the upper bound for $\empiricaldiff^{(\epoch-1)}$ 
in \Cref{lem:upper bound}, 
we thus get
\begin{align*}
    \empiricaldiff^{(\epoch)} 
    & \geq \truediff - 2
    \left(\frac{4\dimension^2 \corruption_\epoch}
    {\totalnum_\epoch}
    + \frac{\empiricaldiff^{(\epoch-1)}}{16}
    \right) \\
    &\geq 
    \truediff - \frac{8\dimension^{2} \corruption_\epoch}{\totalnum_\epoch} 
    - \frac{1}{4}\mlr{\truediff + 2^{-(\epoch-1)}
    + 4\dimension^{2}\sum_{s=1}^{\epoch-1} \slr{\frac{1}{8}}^{m-1-s}\frac{\corruption_s}{\totalnum_s}}\\
    &\geq \frac{\truediff}{2} 
    - 2^{-\epoch-1}
    - 8\dimension^{2}\sum_{s=1}^{\epoch} \slr{\frac{1}{8}}^{m-s}\frac{\corruption_s}{\totalnum_s}. \qedhere
\end{align*}
\end{proof}

\section{Regret estimation}
\label{sec:regret}

\begin{theorem}
\label{thm:final regret}
 With probability at least $1-\confidence$, the regret is bounded by 
\begin{align*}
    \regret = \bigO{
    \frac{\dimension^{2} \corruption \log \horizon}{\truediff} 
    + \frac{\dimension^5 
    \log (\sfrac{\dimension \log \horizon}{\delta})\log \horizon}
    {\truediff^2}}.
\end{align*}
\end{theorem}

\begin{proof}
Let $\regret_1^{(\epoch)}$ and $\regret_2^{(\epoch)}$ be the pseudo regret
for exploitation and exploration in epoch $\epoch$ respectively. 
By \Cref{lem:confidence interval}, 
the event $\event$ occurs with probability $1-\delta$.
We propose first the pseudo regret bound for exploitation given the occurrence of $\event$.

\noindent\textbf{Exploitation:} The pseudo regret for exploitation 
in epoch $\epoch$ is $\regret_1^{(\epoch)} = \num_\epoch\truediff^{(\epoch)}$.

Let $\truediff^{(\epoch)} = \inner{\hidden}{\bestcontext - \context_*^{(\epoch-1)}}$ 
be the pseudo regret for the action $\context_*^{(\epoch-1)}$. 
Given that the event $\event$ happens, we have
\begin{align}\label{eq:bound truediff}
    \truediff^{(\epoch)} 
    = \inner{\hidden - \empirical^{(\epoch-1)}}{\bestcontext} 
    + \inner{\empirical^{(\epoch-1)}}{\bestcontext - \context_*^{(\epoch-1)}} 
    + \inner{\empirical^{(\epoch-1)}-\hidden}{\context_*^{(\epoch-1)}}
    \leq 2\estimationdiff_{\epoch-1}, 
\end{align}
because 
$\inner{\empirical^{(\epoch-1)}}
{\bestcontext - \context_*^{(\epoch-1)}}
\leq 0$. 
Define $\sumerror_{\epoch} 
= \dimension^{2}\sum_{s=1}^{\epoch} 
\slr{\frac{1}{8}}^{m-s}\frac{\corruption_s}{\totalnum_s}$. 
Then we can get
\begin{align}\label{eq:bound truediff 2}
    \truediff^{(\epoch)} \leq \, &
    \frac{8\dimension^{2}\corruption_{\epoch-1}}{\totalnum_{\epoch-1}} 
    + \frac{\empiricaldiff^{(\epoch-2)}}{8} \nonumber\\
    \leq \, & \frac{8\dimension^{2}\corruption_{\epoch-1}}{\totalnum_{\epoch-1}} 
    + \frac{1}{4}\slr{\truediff + 2^{-\epoch+2}+ 4\sumerror_{\epoch-2}} \nonumber\\
    = \, & \frac{\truediff}{4} 
    + 2^{-\epoch} + 8\sumerror_{\epoch-1},
\end{align}
where the first inequality holds by the definition of $\beta_\epoch$ 
and Inequality \eqref{eq:bound truediff}, 
and the second inequality holds by \Cref{lem:upper bound}. 

If $\truediff^{(\epoch)} = 0$, then the total regret for exploitation $\regret_1^{(\epoch)}$ is $0$; 
otherwise, we have $\truediff^{(\epoch)} \geq \truediff$. 
Now we consider two different cases. 

For the case $\truediff \geq 2^{-\epoch+1}$,
we have 
$\frac{\truediff}{4} + 2^{-\epoch}
\leq \frac{\truediff}{4} + \frac{\truediff}{2}
\leq \frac{3\truediff^{(\epoch)}}{4}$. 
Combining it with Inequality \eqref{eq:bound truediff 2},
we have 
$\truediff^{(\epoch)} \leq 32\sumerror_{\epoch-1}$. 
So, the pseudo regret $\regret_1^{(\epoch)}$ is 
\begin{align*}
    \regret_1^{(\epoch)} 
    = \num_\epoch\truediff^{(\epoch)} 
    \leq 32\parameter \cdot 4^\epoch \sumerror_{\epoch-1}.
\end{align*}

For the case $\truediff < 2^{-\epoch+1}$, 
by Inequality \eqref{eq:bound truediff 2},
we have 
$\truediff^{(\epoch)} \leq 
8 \sumerror_{\epoch-1} + 2^{-\epoch+1}$. 
It follows that
\begin{align*}
    \regret_1^{(\epoch)} = \num_\epoch\truediff^{(\epoch)} 
    \leq 8 \parameter \cdot 4^\epoch \sumerror_{\epoch-1} 
    + \parameter \cdot 2^{\epoch+1}
    \leq 8 \parameter \cdot 4^\epoch \sumerror_{\epoch-1} 
    + \frac{4\parameter}{\truediff}.
\end{align*}
Thus, for each epoch $\epoch$, 
\begin{align*}
    \regret_1^{(\epoch)} \leq 
    32\parameter \cdot 4^\epoch \sumerror_{\epoch-1} 
    + \frac{4\parameter}{\truediff}.
\end{align*}
Summing over all epochs yields
\begin{align}\label{eq:r1}
\regret_1 &\leq \frac{4\parameter M}{\truediff}
+ 32 \parameter \sum_{m=1}^M \sumerror_{\epoch-1} 4^\epoch 
\leq 
\frac{4\parameter M}{\truediff} 
+ 32 \parameter 
\sum_{m=1}^M \sum_{s=1}^m \frac{\corruption_s}{8^{\epoch-1-s}\totalnum_s}4^{\epoch} \nonumber\\
&\leq 
\frac{4\parameter M}{\truediff} 
+ 32 \sum_{m=1}^M 
\sum_{s=1}^m \corruption_s \cdot
\frac{4^{\epoch-s}}{8^{m-1-s}} \nonumber\\
&= 
\frac{4\parameter M}{\truediff} 
+ 32 \sum_{s=1}^M \corruption_s 
\sum_{m=s}^M \frac{4^{\epoch-s}}{8^{m-1-s}} 
\leq \frac{4\parameter M}{\truediff} + 512 C, 
\end{align}
where the third inequality holds because $\totalnum_s \geq 4^s$ by the construction of our algorithm. 

\noindent\textbf{Exploration:} 
Now we turn to the exploration part and propose a bound for the pseudo regret  $\regret_2^{(\epoch)}$ in each epoch~$\epoch$.
Note that the expected number of time steps in which
exploration is conducted is 
$\frac{\parameter}{(\empiricaldiff^{(\epoch)})^{2}}$,
and the pseudo regret for each of such time step is bounded above by 1. 

When $\truediff \leq 2^{1-\epoch}$, 
since $\empiricaldiff^{(\epoch)} \geq 2^{-\epoch}$, 
we have 
\begin{align*}
    \regret_2^{(\epoch)} \leq
\frac{\parameter}{(\empiricaldiff^{(\epoch)})^{2}}
\leq \frac{4\parameter}{\truediff^2}. 
\end{align*}

When $\truediff > 2^{1-\epoch}$, 
we again consider two cases.
For the case $\sumerror_{\epoch} \geq \frac{\truediff}{64}$, 
since $\totalnum_\epoch \geq \num_\epoch = \parameter \cdot 4^\epoch$ 
and because $\empiricaldiff^{(\epoch)} \geq 2^{-\epoch}$, 
we have $\totalnum_\epoch \leq \num_\epoch + \parameter (\empiricaldiff^{(\epoch)})^{-2} \leq 2\parameter \cdot 4^\epoch$, 
and 
\begin{align*}
    \frac{\truediff}{64} 
    \leq \sumerror_{\epoch} 
    = \dimension^{2}\sum_{s=1}^{\epoch}
    \slr{\frac{1}{8}}^{\epoch-s} \frac{C_s}{N_s} 
    \leq \frac{2\dimension^{2} 
    \sum_{s=1}^{\epoch} \corruption_s}{\totalnum_\epoch} 
    \leq \frac{2\dimension^{2} 
    \corruption}{\totalnum_\epoch}, 
\end{align*}
where the second inequality holds because 
$\slr{\frac{1}{8}}^{\epoch-s} \frac{1}{N_s} 
\leq \slr{\frac{1}{8}}^{\epoch-s} \frac{1}{\parameter 4^s} 
\leq \frac{1}{\parameter4^m} \leq \frac{2}{N_m}$. 
Thus, the total amount of corruption satisfies 
$\corruption \geq \frac{\totalnum_\epoch \truediff}{128 \dimension^{2}}$, 
which implies that $\regret_2^{(\epoch)} \leq \totalnum_\epoch 
\leq \frac{128 \dimension^{2}\corruption }{\truediff}$. 

For the case $\sumerror_\epoch < \frac{\truediff}{64}$, 
\Cref{lem:lower bound} ensures that $\empiricaldiff^{(\epoch)} \geq \frac{\truediff}{2} - 2^{-\epoch-1} - 8\sumerror_\epoch \geq
\frac{\truediff}{2} - \frac{\truediff}{4} - \frac{\truediff}{8}
=\frac{\truediff}{8}$. 
Then, the regret in epoch $\epoch$ for exploration is 
$\regret_2^{(\epoch)} \leq 
\frac{\parameter}{(\empiricaldiff^{(\epoch)})^{2}}
\leq \frac{64 \parameter}{\truediff^2}$. 

Therefore, the corresponding total pseudo regret for exploration is 
\begin{align}\label{eq:r2}
    \regret_2 = \sum_{m=1}^M \regret_2^{(\epoch)}
    \leq \sum_{m=1}^M 
    \frac{\parameter}{(\empiricaldiff^{(\epoch)})^{2}}
    \leq \frac{64 \parameter \log \horizon}{\truediff^2} 
    + \frac{128 \dimension^2 \corruption \log \horizon}{\truediff}.
\end{align}
Combining Inequalities \eqref{eq:r1} and \eqref{eq:r2}, 
and noting that 
$\parameter = \bigO{\dimension^5 
\log (\sfrac{4\dimension \log \horizon}{\delta})}$,
the total pseudo regret is
\begin{equation*}
    \regret = \regret_1+\regret_2 
    = \bigO{
    \frac{\dimension^{2} \corruption \log \horizon}{\truediff} 
    + \frac{\dimension^5 
    \log (\sfrac{\dimension \log \horizon}{\delta})\log \horizon}
    {\truediff^2}}. \qedhere
\end{equation*}
\end{proof}

\section{Computational efficiency}
\label{sec:efficiency}

According to ~\citet{lovasz1990geometric}, 
there exists a polynomial time algorithm for finding 
a weak L{\"o}wner-John ellipsoid 
$\ellipsoid \subseteq \xSet$. 
 
\begin{theorem}
\label{thm:john theorem}
For any bounded convex body $K \subseteq \reals^\dimension$,
there is a polynomial time algorithm that computes an ellipsoid $E$ satisfies
\begin{align*}
    \ellipsoid \subseteq K \subseteq 2\dimension^{\sfrac{3}{2}}\ellipsoid.
\end{align*}
\end{theorem}

By plugging the polynomial time algorithm
for finding John's ellipsoid
from \citet{lovasz1990geometric} 
into our \Cref{alg:explobycoor}, 
and setting the parameter $\parameter = 2^{14} \dimension^6
\log (\sfrac{4\dimension \log \horizon}{\delta})$, 
we have an computationally efficient algorithm whose 
regret is of 
$\bigO{\frac{\dimension^{5/2} \corruption \log \horizon}{\truediff} 
+ \frac{\dimension^6 
\log (\sfrac{\dimension \log \horizon}{\delta})
\log \horizon}{\truediff^2}}$. 

\section{Conclusion and open problems}
\label{sec:conclusion}

We provide the first algorithm to deal with the stochastic linear optimization with adversarial corruption. 
With probability $1 - \delta$, our algorithm achieves a regret bound of $\bigO{\frac{\dimension^{5/2} \corruption \log \horizon}{\truediff}
+ \frac{\dimension^6 
\log (\sfrac{\dimension \log \horizon}{\delta})
\log \horizon}{\truediff^2}}$, which increases linearly in the total amount of corruption. 
By setting $\delta = \sfrac{1}{T}$, the expected regret of our algorithm yields 
$\bigO{\frac{\dimension^{5/2} \corruption \log \horizon}{\truediff}
+ \frac{\dimension^6 
\log^2\horizon}{\truediff^2}}$. 
Compared to the lower bound given by
~\cite{lattimore2017end} in the no-corruption setting, 
our expected regret only loses an extra multiplicative factor of $\log T$ asymptotically.
It is not clear to us whether this $\log \horizon$ gap is necessary.

Another interesting problem is to extend our model to the linear contextual bandit problem. 
The main challenge here is how one can estimate the actual gap in expected reward when the decision set varies over time.


\newpage
\bibliography{ref}

\appendix

\section*{Appendix}

\section{Concentration Inequality}
\label{apx:concentration}
\begin{lemma}[\citet{hoeffding1963probability}]
\label{lem:concentration}
For any $n$, 
let $\rv_1, \dots, \rv_n$ be independent 
sub-Gaussian random variables 
with variance proxy $\sigma^2$. 
For any $\epsilon > 0$, 
we have 
\begin{align*}
\Pr \mlr{\frac{1}{n} 
\abs{\sum_{i=1}^n (\rv_i - \expect{\rv_i})} \geq \epsilon}
\leq 2 \exp\blr{-\frac{n \epsilon^2}{2 \sigma^2}}. 
\end{align*}
Moreover, if $\rv_1, \dots, \rv_n$ are 
independent random variables bounded in $[-1, 1]$, 
then for any $\epsilon > 0$, 
we have 
\begin{align*}
\Pr \mlr{
\abs{\sum_{i=1}^n (\rv_i - \expect{\rv_i})} 
\geq \epsilon}
\leq 2 \exp\blr{-\frac{\epsilon^2}
{3\sum_{i=1}^n \expect{\rv_i} + \epsilon}}. 
\end{align*}
\end{lemma}

\end{document}